\documentclass[conference,10pt]{ieeeconf}
\IEEEoverridecommandlockouts

\usepackage[utf8]{inputenc} 
\usepackage[T1]{fontenc}    
\usepackage{url}            
\usepackage{booktabs}       
\usepackage{amsfonts}       
\usepackage{nicefrac}       
\usepackage{microtype}      
\usepackage{lipsum}		
\usepackage{graphicx}
\usepackage{doi}

\usepackage{titlesec}

\usepackage{subfig}

\usepackage{array}
\usepackage{pdfpages}
\usepackage{booktabs}
\usepackage{amsmath,amsfonts,amssymb,dsfont}
\usepackage{stmaryrd}
\usepackage{nccmath}

\usepackage{xcolor}
\usepackage{hyperref}
\hypersetup{
    colorlinks=true,
    linkcolor=blue,
    filecolor=magenta,      
    citecolor=blue,
    urlcolor=cyan,
    }

\usepackage{cite}

\newcommand{\appref}[1]{\hyperref[#1]{App.~\ref*{#1}}}

\usepackage{mathtools}
\mathtoolsset{showonlyrefs=true}

\usepackage[standard,amsmath]{ntheorem}

\theorembodyfont{\upshape}
\newtheorem{assumption}{Assumption}

\newcommand\ALPHABET{\mathsf}
\newcommand\reals{\mathds{R}}
\newcommand\integers{\mathds{Z}}
\newcommand\naturalnumbers{\mathds{N}}

\DeclareMathOperator*{\argmax}{arg\,max}

\newcommand\PR{\mathds{P}}
\newcommand\EXP{\mathds{E}}
\newcommand\IND{\mathds{1}}
\newcommand\GRAD{\nabla}

\DeclarePairedDelimiter{\NORM}{\lVert}{\rVert}
\DeclarePairedDelimiter\MOD{\llbracket}{\rrbracket}

\newcommand*\sublabel[1]{\textsc{\MakeLowercase{#1}}}

\newcommand*\piexpl{\mu}

\newcommand*\pref{p_{\sublabel{REF}}}
\newcommand*\zexpl{\zeta_{\piexpl}}

\newcommand*\reg{\Omega}
\newcommand*\regconj{\Omega^{\star}}
\newcommand*\Bellman{\mathcal{B}}
\newcommand*\regBellman{\mathcal{B}_{\piexpl}}
\newcommand*\periodicregBellman{\mathbb{B}_{\piexpl}}

\newcommand\MATRIX[1]{\begin{bmatrix}#1\end{bmatrix}}

\usepackage{pifont}

\title{\LARGE \bf Convergence of regularized agent-state-based Q-learning in POMDPs}
\author{
Amit Sinha,
Matthieu Geist, Aditya Mahajan%
\thanks{A. Sinha and A. Mahajan are with the Department of Electrical and Computer Engineering, McGill University, Canada. Emails: amit.sinha@mail.mcgill.ca, aditya.mahajan@mcgill.ca. Their work was supported in part by in part by a grant from Google’s Institutional Research Program in collaboration with Mila.
M. Geist is with the Earth Species Project. Email: matthieu@earthspecies.org}%
}

\begin{document}
\maketitle


\begin{abstract}
In this paper, we present a framework to understand the convergence of commonly used Q-learning reinforcement learning algorithms in practice. Two salient features of such algorithms are: (i)~the Q-table is recursively updated using an agent state (such as the state of a recurrent neural network) which is not a belief state or an information state and (ii)~policy regularization is often used to encourage exploration and stabilize the learning algorithm. We investigate the simplest form of such Q-learning algorithms which we call regularized agent-state-based Q-learning (RASQL) and show that it converges under mild technical conditions to the fixed point of an appropriately defined regularized MDP, which depends on the stationary distribution induced by the behavioral policy. We also show that a similar analysis continues to work for a variant of RASQL that learns periodic policies. We present numerical examples to illustrate that the empirical convergence behavior matches with the proposed theoretical limit.
\end{abstract}

\section{Introduction}

Reinforcement learning (RL) is a useful paradigm in learning optimal control policies via simulation when the system model is not available or when the system is too large to explicitly solve the dynamic program. 
The simplest setting is the fully-observed setting of Markov decision processes (MDP), where the controller has access to the environment state. Most existing theoretical RL results on convergence of learning algorithms and their rates of convergence and regret bounds, etc.\ are established for the MDP setting.


However, in many real-world applications, such as
autonomous driving, robotics, healthcare, finance, and others,
the controller does not have access to the environment state;
rather, it has a partial observation of the environment state. So these applications need to be modeled as a partially observable Markov decision process (POMDP) rather than a MDP.


When the system model is known, the POMDP model can be converted into an MDP by considering the controller's belief on the state of the environment (also called the belief state) as an information state~\cite{aastrom1965optimal,smallwood1973optimal,subramanian2022approximate}.
However, such a reduction does not work in the RL setting because the belief state depends on the system model, which is unknown. Nonetheless, there have been several empirical works which show that standard RL algorithms for MDPs continue to work for POMDPs if one uses ``frame stacking'' (i.e., use the last few observations as a state) or recurrent neural networks~\cite{hausknecht2015deep, igl2018deep, zhu2017improving, meng2021memory}. In recent years, considerable progress has been made in understanding the properties of such algorithms but a complete theoretical understanding is still lacking. 


A common way to model such RL algorithms for POMDPs is to consider the state of the controller as an agent state~\cite{dong2022simple}. Such agent-state-based-controllers have also been considered in the planning setting as they can be simpler to implement than belief-state based controllers. See~\cite{sinha2024agent} for an overview. 

A challenge in understanding the convergence of agent-state-based RL algorithms for POMDPs is that an agent state is not an information state. So, it is not possible to write a dynamic programming decomposition based on the agent state. So, one cannot follow the typical proof techniques used to evaluate the convergence of RL algorithms for MDPs (where RL algorithms can be viewed as stochastic approximation variant of MDP algorithms such as value iteration and policy iteration to compute the optimal policy). 

There is a good understanding of the convergence of agent-state-based Q-learning (ASQL) for POMDPs~\cite{Jaakkola1994,Kara2022,sinha2024periodic} (which is related to Q-learning for non-Markovian environments~\cite{Kara2024,chandak2024reinforcement}). There is also some work on understanding the convergence of actor-critic algorithms for POMDPs~\cite{konda2003onactor,subramanian2022approximate}. However, most practical RL algorithms for POMDPs use some form of policy regularization, while most theoretical analysis is restricted to the unregularized setting.
\looseness=-1


Regularization adds an auxiliary loss to the per-step rewards. This loss typically depends on the policy but may also depend on the value function. Regularization is commonly used in RL algorithms for various reasons, such as entropy regularization to encourage exploration~\cite{schulman2017proximal,haarnoja2018sac,ahmed2019understanding} and improve generalization~\cite{henderson2018deep}, KL-regularization to constrain the policy updates to be similar to a prior policy~\cite{peters2010relative,schulman2015trust}, and others.
Unified theory for different facets of regularization in MDPs is provided in~\cite{neu2017unified,geist2019theory}.




Based on the various benefits of regularization in RL for MDPs, it is also commonly used in RL for POMDPs~\cite{igl2018deep, schulman2015trust, schulman2017proximal, hafner2019learning, zhang2019solar, hafner2019dream, subramanian2022approximate, ni2024bridging}. However, the recent theoretical analysis of RL for POMDPs discussed above do not consider regularization. The objective of this paper is to present initial results on understanding regularization in RL for POMDPs. 

There is some recent work on understanding regularization in POMDPs but they either consider the role of entropy regularization in POMDP solvers (when the model information  is known)~\cite{delecki2024entropy,somani2013despot}, or consider regularization of the belief distribution~\cite{molloy2023smoother} or observation distribution~\cite{zamboni2024limits}. These results do not directly provide an understanding of the role of regularization in RL for POMDPs.



In this paper, we revisit Q-learning for POMDPs when the learning agent is using an agent state and using policy regularization. Our main contribution is to show that in this setting, Q-learning converges under mild technical conditions. We characterize the converged limit in terms of the model parameters and choice of behavioral policy used in Q-learning. Recently, it has been argued that periodic policies may perform better when considering agent-state-based POMDPs~\cite{sinha2024periodic}. We show that our analysis extends to a periodic version of regularized Q-learning as well.

\textit{Notation:}
We use uppercase letters to denote random variables (e.g. $S,A$, etc.), lowercase letters to denote their realizations (e.g. $s,a$, etc.) and calligraphic letters to denote sets (e.g. $\ALPHABET{S}, \ALPHABET{A}$; etc.). Subscripts (e.g. $S_t,A_t$, etc.) denote variables at time~$t$. Similarly, $S_{1:t}$ denotes the collection of random variables from time $1$ to $t$. $\Delta (\ALPHABET{S})$ denotes the space of probability measures on a set $\ALPHABET{S}$; $\PR(\cdot)$ and $\EXP[\cdot]$ denote the probability of an event and the expectation of a random variable, respectively; and $\IND(\cdot)$ denotes the indicator function. $\lvert \ALPHABET S \rvert$ denotes the number of elements in $\ALPHABET S$ (when it is a finite set). $\reals$ denotes real numbers. $[L]$ denotes the set of integers from $0$ to $L-1$, where $L \in \integers^{+}$. $\MOD{\ell}$ denotes $(\ell \text{ mod } L)$.

\section{Background}

\subsection{Legendre-Fenchel transform (convex conjugate)}

We start with a short review of convex conjugates and Legendre-Fenchel transforms~\cite{rockafellar2009variational}, which are an important tool to understand regularization in MDPs~\cite{geist2019theory}.

\begin{definition} \label{def:LF-transform}
    For a strongly convex function $\reg \colon \reals^n \to \reals$, its convex conjugate $\regconj \colon \reals^n \to \reals$ is defined as
    \begin{align} \label{eq:convex-conj-objective}
        \regconj(q) = \max_{p \in \reals^n} \bigl\{ \langle p, q \rangle - \reg(p) \bigr\}.
    \end{align}
    The mapping $\Omega \mapsto \Omega^*$ is the Legendre-Fenchel transform.
\end{definition}

The following is a useful property of the Legendre-Fenchel transform for regularized MDPs:

\begin{lemma}[Based on \cite{hiriart2004fundamentals, mensch2018differentiable}] \label{lemma:opt-reg-pol}
Let $\Delta$ be a simplex in $\reals^n$ and $\reg \colon \Delta \to \reals$ be twice differentiable and a strongly convex function. Let $\regconj \colon \reals^n \to \reals$ be the Legendre-Fenchel transform of $\Omega$. Then, $\GRAD \regconj$ is Lipschitz and satisfies
\begin{equation} \label{eq:opt-pol-convex-conj}
    \GRAD \regconj (q) = \argmax_{p \in \Delta} \bigl\{ \langle p, q \rangle  - \reg(p) \bigr\}.
\end{equation}
\end{lemma}

In Markov decision processes, one often regularizes the policy. Below we describe some of the commonly used policy regularizers. For the purpose of the discussion below, let $\ALPHABET A$ be a finite set (later we will take $\ALPHABET A$ to be the set of actions of an MDP, but for now we can consider it as a generic set). 
\begin{enumerate}
    \item \textbf{Entropy regularization} uses the regularizer $\reg \colon \Delta(\ALPHABET A) \to \reals$ given by
    \(
        \reg(p) = \frac{1}{\beta} \medop\sum_{a \in \ALPHABET A} p(a) \ln p(a)
    \)
    where $\beta \in \reals_{> 0}$ is a parameter. Its convex conjugate $\regconj \colon \reals^{|\ALPHABET A|} \to \reals$ is given by
    \(
        \regconj(q) = \frac{1}{\beta}  \ln \bigl(\medop\sum_{a \in \ALPHABET A} \exp (\beta q(a))\bigr).
    \)
    Furthermore, from Lemma~\ref{lemma:opt-reg-pol}, we get that the argmax in the definition of convex conjugate is achieved by
    \begin{equation*}
        p^{\star} (a) = \frac{\exp (\beta q(a))}{\sum_{a' \in \ALPHABET A} \exp (\beta q(a'))}.
    \end{equation*}

    \item \textbf{KL regularization} uses the regularizer $\reg \colon \Delta(\ALPHABET A) \to \reals$ given by
    \(
        \reg(p) = \frac{1}{\beta} \medop\sum_{a \in \ALPHABET A} p(a) \ln ({p(a)}/{\pref(a)}),
    \)
    where $\beta \in \reals_{> 0}$ is a parameter and $\pref \in \Delta(\ALPHABET A)$ is a reference distribution. Its convex conjugate $\regconj \colon \reals^{|\ALPHABET A|} \to \reals$ is given by
    \(
        \regconj(q) = \frac{1}{\beta} \ln \bigl(\medop\sum_{a \in \ALPHABET A} \pref(a) \exp (\beta q(a))\bigr).
    \)
    Furthermore, from Lemma~\ref{lemma:opt-reg-pol}, we get that the argmax in the definition of convex conjugate is achieved by
    \begin{equation*}
     p^{\star} (a) = \frac{\pref(a) \exp (\beta q(a))}{\sum_{a' \in \ALPHABET A} \pref(a') \exp (\beta q(a'))}.
    \end{equation*}
\end{enumerate}

\subsection{Regularized MDPs} \label{subsec:reg-mdps}

In this section, we provide a brief review of regularized Markov decision processes (MDPs), which are a generalization of standard MDPs with an additional ``regularization cost'' at each stage. 

Consider a Markov decision process (MDP) with state $s_t \in \ALPHABET S$, control action $a_t \in \ALPHABET A$, where all sets are finite. The system operates in discrete time. The initial state $s_1 \sim \rho$ and for any time $t \in \naturalnumbers$, we have
\begin{align*}
    \PR(s_{t+1} \mid s_{1:t}, a_{1:t}) = \PR(s_{t+1} \mid s_{t}, a_{t}) \eqqcolon P(s_{t+1} \mid s_t, a_t),
\end{align*}
where $P$ is a probability transition matrix. The system yields a reward $R_t = r(s_t, a_t) \in [0, R_{\max}]$. The rewards are discounted by a factor $\gamma \in [0,1)$. 

Consider a policy $\pi : \ALPHABET S \to \Delta(\ALPHABET A)$. Let $\reg \colon \Delta(\ALPHABET A) \to \reals$ be a strongly convex function that is used as a policy regularizer. Then, the \emph{regularized performance} of policy $\pi$ is given by
\begin{equation} \label{eq:mdp-performance}
    J^{\reg}_{\pi} \coloneqq \EXP^{\pi}
    \biggl[ \medop\sum_{t=1}^\infty \gamma^{t-1} \bigl[ r(s_t, a_t) - \reg(\pi(\cdot \mid s_t)) \bigr] \biggm| s_1 \sim \rho \biggr],
\end{equation}
where the notation $\EXP^{\pi}$ means that the expectation is taken with the joint measure on the system variables induced by the policy $\pi$.

The objective in a regularized MDP is to find a policy $\pi$ that maximizes the regularized performance $J^{\reg}_{\pi}$ defined above. A key step in understanding the optimal solution of the regularized MDP is to define the regularized Bellman operator $\Bellman^{\reg}$ on the space of real-valued functions on $\ALPHABET S \times \ALPHABET A$ as follows. For any $Q : \ALPHABET S \times \ALPHABET A \to \reals$, 
\begin{equation}
    \Bellman^{\reg} Q(s, a) = r (s, a) + \gamma \sum_{s' \in \ALPHABET S} P (s' \mid s,a) \regconj(Q(s', \cdot)),
\end{equation}
where $\regconj$ is the Legendre-Fenchel transform of $\reg$.

\begin{proposition}[Based on \cite{geist2019theory}] \label{prop:reg-mdp-result}
The following results hold:
\begin{enumerate}
    \item The operator $\Bellman^{\reg}$ is a contraction and therefore has a unique fixed point, which we denote by $Q^{\reg}$. By definition, 
    \begin{equation}
        Q^{\reg}(s, a) = r(s, a)  +  \gamma \sum_{s' \in \ALPHABET S} P(s' \mid s,a) \regconj(Q^{\reg}(s', \cdot)).
        \label{eq:mdp-dp-q} 
    \end{equation}

    \item Define the policy $\pi^{\reg,*} \colon \ALPHABET S \to \Delta(\ALPHABET A)$ as follows: for any $s \in \ALPHABET S$, 
   \begin{align}
        \pi^{\reg, \star}(\,\cdot \mid s) &=  \GRAD \regconj(Q^{\reg}(s, \cdot))
        \notag \\
        &=
        \argmax_{\xi \in \Delta(\ALPHABET A)} \left\{ \sum_{a \in \ALPHABET A} \xi(a) Q^{\reg} (s, a) - \reg(\xi) \right\}
        \label{eq:regconj-based-pi-star}
    \end{align}
    where the last equality follows from Lemma~\ref{lemma:opt-reg-pol}. Then, the policy $\pi^{\reg,\star}$ is optimal for maximizing  the regularized performance $J^{\reg}_{\pi}$ over the set of all policies.
 \end{enumerate}
\end{proposition}


\section{System model and regularized Q-learning for POMDPs}

\subsection{Model for POMDPs}
Consider a partially observable Markov decision process (POMDP) with state $s_t \in \ALPHABET S$, control action $a_t \in \ALPHABET A$, and output $y_t \in \ALPHABET Y$, where all sets are finite. The system operates in discrete time with the dynamics given as follows. The initial state $s_1 \sim \rho$ and  for any time $t \in \naturalnumbers$, we have
\begin{align*}
    \PR(s_{t+1}, y_{t+1} \mid s_{1:t}, y_{1:t}, a_{1:t}) &= \PR(s_{t+1}, y_{t+1} \mid s_{t}, a_{t}) \\
    &\eqqcolon P(s_{t+1}, y_{t+1} \mid s_t, a_t),
\end{align*}
where $P$ is a probability transition matrix. In addition, at each time the system yields a reward $r_t = r(s_t, a_t) \in [0, R_{\max}]$. The rewards are discounted by a factor $\gamma \in [0,1)$. 

Let $\boldsymbol{\vec {\pi}} = (\vec \pi_1, \vec \pi_2, \dots)$ denote any (history dependent and possibly randomized) policy, i.e., under policy $\boldsymbol{\vec \pi}$ the action at time~$t$ is chosen as $a_t \sim \vec \pi_t(y_{1:t}, a_{1:t-1})$. The performance of policy $\boldsymbol{\vec {{\pi}}}$ is given by 
\begin{equation} \label{eq:pomdp-performance}
    J_{\boldsymbol{\vec{{\pi}}}} \coloneqq \EXP^{\boldsymbol{\vec{{\pi}}}}
    \biggl[ \medop\sum_{t=1}^\infty \gamma^{t-1} r(s_t, a_t)  \biggm| s_1 \sim \rho \biggr],
\end{equation}
where the notation $\EXP^{\boldsymbol{\vec \pi}}$ means that the expectation is taken with the joint measure on the system variables induced by the policy $\boldsymbol{\vec \pi}$.

The objective is to find a (history dependent and possibly randomized) policy $\boldsymbol{\vec{{\pi}}}$ to maximize $J_{\boldsymbol{\vec{{\pi}}}}$. When the system model is known, the above POMDP model can be converted to a fully observed Markov decision process (MDP) by considering the controller's posterior belief on the system state as an information state~\cite{aastrom1965optimal, smallwood1973optimal}. However, when the system model is not known, it is not possible to run reinforcement learning (RL) algorithms on the belief-state MDP because the belief depends on the system model. For that reason, in RL for POMDPs it is often assumed that the controller is an agent-state-based controller. 

\begin{definition}[Agent state]
An agent state is a model-free recursively updateable function of the history of observations and actions. In particular, let $\ALPHABET Z$ denote the agent state space. Then, the agent state is a process $\{z_t\}_{t \ge 0}$, $z_t \in \ALPHABET Z$,  which starts with some initial value $z_0$, and is then recursively computed as
\begin{equation} \label{eq:agent-state-update}
    z_{t+1} = \phi(z_t, y_{t+1}, a_t), \quad t \ge 0
\end{equation}    
where $\phi$ is a pre-specified agent-state update function. 
\end{definition}
Some examples of agent-state-based controllers are: (i)~a finite memory controller, which chooses the actions based on the previous $k$ observations; (ii)~a finite state controller, which effectively filters the possible histories to values from a finite set $\ALPHABET Z$. We refer the reader to~\cite{sinha2024agent} for a detailed review of agent-state-based policies in POMDPs.

We use $\boldsymbol {\pi} = (\pi_1, \pi_2, \dots)$ to denote an agent-state-based policy,\footnote{We use $\boldsymbol{\vec{\pi}}$ to denote history dependent policies and $\boldsymbol{\pi}$ to denote agent-state-based policies.} i.e., a policy where the action at time~$t$ is given by $a_t \sim \pi_t(z_t)$. 
An agent-state-based policy is said to be \textbf{stationary} if for all $t$ and $t'$, we have $\pi_t(a \mid z) = \pi_{t'}(a \mid z)$ for all $(z,a) \in \ALPHABET Z \times \ALPHABET A$.

If the agent state is an information state, then MDP-based RL algorithms can directly be applied to find optimal stationary solutions \cite{subramanian2022approximate}. However, in general, an agent state is not an information state, as is the case in frame-stacking or when using recurrent neural networks. In such settings, the dynamics of the agent state process is non Markovian and the standard dynamic programming based argument does not work. It is possible to find the optimal policy by viewing the POMDP with an agent-state-based controller as a decentralized control problem and using the designer's approach~\cite{MahajanPhD} to compute an optimal agent-state-based policy, as is done in~\cite{sinha2024agent}, but such an approach is intractable for all but small toy problems.


The Q-learning algorithms for POMDPs maintain a Q-table based on the agent states and actions and update the Q-values based on the samples generated by the environment. Since the agent state is non Markovian, it is not clear if such an iterative scheme converges, and if so, to what value. In the next section, we present a formal model for agent state based Q-learning when the agent also uses policy regularization.


\subsection{Regularized agent-state-based Q-learning for POMDPs} \label{eq:reg-MDP}

In this section we describe regularized agent-state-based Q-learning (RASQL), which is an online off-policy learning approach in which the agent acts according to a fixed behavioral policy to generate a sample path $(z_1, a_1, r_1, z_2, \dots)$ of agent states, actions, and rewards observed by a learning agent. We assume that the sampled rewards $r_t = r(s_t, a_t)$ are available to the agent during the learning process.


The learning agent uses a policy regularizer $\reg \colon \Delta(\ALPHABET A) \to \reals$ and
maintains a regularized Q-table, which is arbitrarily initialized and then recursively updated as follows:
\begin{align} \label{eq:reg-QL}
    &Q_{t+1} (z, a) = Q_{t} (z, a) \nonumber \\
    &\quad + \alpha_t(z, a) \left[ r_t + \gamma \regconj(Q_t(z_{t+1}, \cdot)) - Q_t(z, a)\right],
\end{align}
where the learning rate sequence $\{\alpha_t(z,a)\}_{t \ge 1}$ is chosen such that $\alpha_t(z,a) = 0$ whenever $(z,a) \neq (z_t,a_t)$. For instance, if the policy regularizer is the entropy regularizer, then the above iteration corresponds to an agent-state-based version of soft-Q-learning \cite{haarnoja2017reinforcement}. The ``greedy'' policy at each time is given by $\pi_t(\cdot \mid z) = \GRAD \regconj(Q_t(z,\cdot))$. Thus, for entropy regularization, it would correspond to soft-max based on $Q_t$.

If the $\regconj(Q_{t}(z_{t+1}, \cdot))$ term in~\eqref{eq:reg-QL} is replaced by $\max_{a' \in \ALPHABET A} Q_{t}(z_{t+1}, a')$, the iteration in RASQL corresponds to agent-state-based Q-learning (ASQL):
\begin{align} \label{eq:ASQL}
    &Q_{t+1} (z_t, a_t) = Q_{t} (z_t, a_t) \quad +\nonumber \\
    &\quad \alpha_t(z_t, a_t) \left[ r_t + \gamma \max_{a' \in \ALPHABET A} Q_t(z_{t+1}, a') - Q_t(z_t, a_t)\right].
\end{align}
The convergence of ASQL and its variations have been recently studied in~\cite{chandak2024reinforcement,sinha2024periodic,Kara2022}. However, the analysis of ASQL does not include regularization. The main result of this paper is to characterize the convergence of RASQL.

\section{Main result}

We impose the following standard assumptions on the model.
\begin{assumption}\label{ass:lr}
    For all $(z,a)$, 
    the learning rates $\{\alpha_t(z,a)\}_{t \ge 1}$ are measurable with respect to the sigma-algebra generated by $(z_{1:t}, a_{1:t})$ and satisfy $\alpha_t(z,a) = 0$ if $(z,a) \neq (z_t,a_t)$. Moreover, 
    \(
        \sum_{t \ge 1} \alpha_t(z,a) = \infty
    \)
    and
    \(
        \sum_{t \ge 1} (\alpha_t(z,a))^2 < \infty
    \), almost surely.
\end{assumption}

\begin{assumption}\label{ass:policy}
    The behavior policy $\piexpl$ is such that the Markov chain $\{(S_t, Y_t, Z_t, A_t)\}_{t \ge 1}$ converges to a limiting distribution $\zexpl$, where $\sum_{(s,y)} \zexpl(s,y,z,a) > 0$ for all $(z,a)$ (i.e., all $(z,a)$ are visited infinitely often).
\end{assumption}

Assumption~\ref{ass:lr} is the standard assumption for convergence of stochastic approximation algorithms \cite{robbins1951stochastic}.
Assumption~\ref{ass:policy} ensures persistence of excitation and is a standard assumption in convergence analysis of Q-learning \cite{watkins1992q, tsitsiklis1994asynchronous, Jaakkola1994,Kara2022,sinha2024periodic}.

For ease of notation, we will continue to use $\zexpl$ to denote the marginal and conditional distributions w.r.t.\ $\zexpl$. In particular, for marginals we use $\zexpl(y,z,a)$ to denote $\sum_{s \in \ALPHABET S} \zexpl(s,y,z,a)$ and so on; for conditionals, we use  $\zexpl(s | z,a)$ to denote $\zexpl(s,z,a) / \zexpl(z,a)$ and so on. Note that $\zexpl(s,z,y,a) = \zexpl(s,z) \piexpl (a | z) P(y|s,a)$. Thus, we have that $\zexpl(s | z, a) = \zexpl(s | z)$.

The key idea to characterize the convergence behavior is the following. Given the limiting distribution $\zexpl$, we can define an MDP with state space $\ALPHABET Z$, action space $\ALPHABET A$, and per-step reward $r_{\piexpl} \colon \ALPHABET Z \times \ALPHABET A \to \reals$ and dynamics $P_{\piexpl} \colon \ALPHABET Z \times \ALPHABET A \to \Delta(\ALPHABET Z)$ given as follows:
\begin{align}\label{eq:artificial-MDP}
    r_{\piexpl}(z,a) &\coloneqq \medop\sum_{s \in \ALPHABET S} r(s,a) \zexpl(s \mid z), \\
    P_{\piexpl}(z'|z,a) &\coloneqq \smashoperator{{\medop\sum}_{ (s,y') \in \ALPHABET S \times \ALPHABET Y}}
    \IND_{\{z' = \phi(z,y',a)\}} P(y'|s,a) \zexpl(s|z).  
\end{align}

Now consider a regularized version of this MDP, where we regularize the policy using $\reg$. Let $Q_{\piexpl}$ denote the fixed point of the regularized Bellman operator corresponding to this regularized MDP, i.e., $Q_{\piexpl}$ is the unique fixed point of the following (see the discussion in Sec.~\ref{subsec:reg-mdps}):
\begin{equation} \label{eq:pseudo-mdp-dp-q} 
    Q_{\piexpl} (z, a) = r_{\piexpl} (z, a)  +  \gamma \sum_{z' \in \ALPHABET Z} P_{\piexpl} (z' \mid z,a) \regconj(Q_{\piexpl}(z', \cdot)).
\end{equation}

Then, our main result is the following:


\begin{theorem} \label{thm:QL-POMDP}
    Under Assumptions~\ref{ass:lr} and~\ref{ass:policy}, the RASQL iteration~\eqref{eq:reg-QL} converges to $Q_{\piexpl}$ almost surely.
\end{theorem}
\begin{proof}
 The proof is given in appendix \ref{sec:RASQL-proof}.
\end{proof}

\begin{remark}
    Note that Proposition~\ref{prop:reg-mdp-result} implies that the ``greedy'' regularized policy with respect to the limit point of $\{Q_t\}_{t\ge 1}$ is given by
    $\pi^*(\cdot \mid z) = \GRAD \regconj(Q_{\piexpl}(z, \cdot))$, which typically lies in the interior of $\Delta(\ALPHABET A)$ for each $z$. Thus, the greedy policy is stochastic. This is a big advantage of RASQL compared to ASQL because in ASQL, the greedy policy corresponding to the limit point of the Q-learning iteration is deterministic. As shown in \cite{singh1994learning} (also see~\cite{sinha2024agent,sinha2024periodic}), in general for POMDPs with agent-state-based controllers, stochastic stationary policies can outperform deterministic stationary policies.
\end{remark}

\section{Regularized periodic Q-learning}

The idea of periodic Q-learning has been explored in \cite{sinha2024periodic}. They show that periodic policies can perform better than stationary policies when the agent state is not an information state. Regularized Q-learning can be generalized by regularized periodic Q-learning, since taking the period $L=1$ reproduces the stationary setting.

Consider the convergence properties when we consider the following regularized periodic agent-state-based Q-learning (RePASQL) update for $\ell \in [L]$.
\begin{align} \label{eq:reg-PASQL}
    &Q^{\ell}_{t+1} (z, a) = Q^{\ell}_{t} (z, a) \nonumber \\
    &\quad + \alpha^{\ell}_t(z, a) \left[ r_t + \gamma \regconj(Q^{\MOD{\ell+1}}_t(z', \cdot)) - Q^{\ell}_t(z, a)\right].
\end{align}

\begin{assumption}\label{ass:periodic-lr}
    For all $(\ell, z,a)$, 
    the learning rates $\{\alpha^{\ell}_t(z,a)\}_{t \ge 1}$ are measurable with respect to the sigma-algebra generated by $(z_{1:t}, a_{1:t})$ and satisfy $\alpha^{\ell}_t(z,a) = 0$ if $(\ell, z,a) \neq (\MOD{t},z_t,a_t)$. Moreover, 
    \(
        \sum_{t \ge 1} \alpha^{\ell}_t(z,a) = \infty
    \)
    and
    \(
        \sum_{t \ge 1} (\alpha^{\ell}_t(z,a))^2 < \infty
    \), almost surely.
\end{assumption}

\begin{assumption}\label{ass:periodic-policy}
    The behavior/exploration policy $\piexpl = \{ \piexpl^{\ell} \}_{\ell \in [L]}$ is such that the Markov chain $\{(S_t, Y_t, Z_t, A_t)\}_{t \ge 1}$ converges to a limiting periodic distribution $\zexpl^{\ell}$, where $\sum_{(s,y)} \zexpl^{\ell}(s,y,z,a) > 0$ for all $(\ell,z,a)$ (i.e., all $(\ell,z,a)$ are visited infinitely often).
\end{assumption}
By considering this limiting distribution w.r.t. the original  model's rewards and dynamics, we can construct an artificial MDP on the agent state for each $\ell \in [L]$, which has the following rewards and dynamics:
\begin{align}\label{eq:artificial-MDP-periodic}
    r^{\ell}_{\piexpl}(z,a) &\coloneqq \medop\sum_{s \in \ALPHABET S} r(s,a) \zexpl^{\ell}(s \mid z), \\
    P^{\ell}_{\piexpl}(z'|z,a) &\coloneqq \smashoperator{{\medop\sum}_{ (s,y') \in \ALPHABET S \times \ALPHABET Y}}
    \IND_{\{z' = \phi(z,y',a)\}} P(y'|s,a) \zexpl^{\ell}(s|z).  
\end{align}

Now we can extend the same techniques used in regularized MDPs \ref{subsec:reg-mdps} to this by defining a regularized Bellman operator $\regBellman^{\ell}$ on an arbitrary Q-function $Q \in \reals^{\lvert \ALPHABET Z \rvert \times \lvert \ALPHABET A \rvert}$ as follows:
\begin{equation}
    \regBellman^{\ell} Q(z, a) = r^{\ell}_{\piexpl} (z, a) + \gamma \sum_{z' \in \ALPHABET Z} P^{\ell}_{\piexpl} (z' \mid z,a) \regconj(Q(z', \cdot)).
\end{equation}

Next define the composition of the sequence of $L$ Bellman operators corresponding to cycle $\ell$ as is done in \cite{sinha2024periodic}.
\begin{equation}
    \periodicregBellman^{\ell} = \regBellman^{\ell} \regBellman^{\MOD{\ell+1}} \cdots \regBellman^{\MOD{\ell+L-1}}.
\end{equation}

Then we can apply Prop.~\ref{prop:reg-mdp-result} to $\periodicregBellman^{\ell}$. In addition, considering the periodicity of the operators, the same approach followed in \cite{sinha2024periodic} can be used to show that $\periodicregBellman^{\ell}$ is a contraction and therefore has a unique fixed point denoted by $Q^{\ell}_{\piexpl}$ which is given by 
    \begin{equation}
    Q^{\ell}_{\piexpl} (z, a) = r^{\ell}_{\piexpl} (z, a) + \gamma \sum_{z' \in \ALPHABET Z} P^{\ell}_{\piexpl} (z' \mid z,a) V^{\MOD{\ell+1}}_{\piexpl} (z'). \label{eq:pseudo-mdp-dp-q-periodic}
    \end{equation}

\begin{theorem} \label{thm:PRASQL}
Under Assumptions~\ref{ass:periodic-lr} and~\ref{ass:periodic-policy}, the RePASQL iteration~\eqref{eq:reg-PASQL} converges to $\{Q^{\ell}_{\piexpl}\}_{\ell \in [L]}$ almost surely.
\end{theorem}

\begin{proof}
The proof is given in appendix \ref{sec:PRASQL-proof}.
\end{proof}
\section{Numerical example}

In this section, we present an example to highlight the salient features of our results. First, we describe the POMDP model.

\subsection{POMDP model}

Consider a POMDP with $\ALPHABET S = \{ 0, 1, 2, 3\}, \ALPHABET A = \{ 0, 1 \}, \ALPHABET Y = \{ 0, 1\}$ and $\gamma = 0.9$. The start state distribution is given by
\[
\rho(s) = \MATRIX{0.3, 0.0, 0.2, 0.5}
\]

Now consider the reward and transitions when $a=0$:
\begin{align*}
	r(s, 0) &= (1-\gamma)\times \MATRIX{0.6, 0.0, 0.5, -0.3} \\
    P(s' \mid s, 0) &= \MATRIX{0.0 & 0.6 & 0.4 & 0.0 \\ 0.8 & 0.0 & 0.2 & 0.0 \\ 0.7 & 0.3 & 0.0 & 0.0 \\ 0.2 & 0.0 & 0.0 & 0.8}.
\end{align*}
Note that $s, s'$ (state, next state) corresponds to the rows, columns of $P$, respectively. Next, when $a=1$
\begin{align*}
    r(s, 1) &= (1-\gamma)\times \MATRIX{0.1, -0.3, -0.2, 0.5} \\
    P(s' \mid s, 1) &= \MATRIX{0.8 & 0.2 & 0.0 & 0.0 \\ 0.4 & 0.0 & 0.6 & 0.0 \\ 0.0 & 0.8 & 0.2 & 0.0 \\ 0.1 & 0.7 & 0.2 & 0.0}.
\end{align*}
Finally, we have the observations function which maps $s=\{0, 3\}$ to $y=0$ and $s=\{1, 2\}$ to $y=1$.

\subsection{Regularized agent-state-based Q-learning (RASQL) experiment}

For the purpose of providing a simple illustration in this example, we fix the agent state to be the observation of the agent, i.e., $z_t=y_t$. However, in general the theoretical results hold for the general agent-state update rule given in \eqref{eq:agent-state-update}. Consider the following fixed exploration policy:
\[
\piexpl(a \mid z) = \MATRIX{0.2 & 0.8 \\ 0.8 & 0.2}.
\]
Note that $z,a$ (observation, action) corresponds to the rows, columns of $\piexpl$, respectively.

Using $\piexpl$, we run $25$ random seeds on the given POMDP and we perform the RASQL update \eqref{eq:reg-QL} with a regularization coefficient $(\beta) = 1.0$ for $10^5$ timesteps/iterations. We plot the median and quartiles from $25$ seeds of the iterates $\{Q_{t}(z, a)\}_{t \geq 1}$ for each $(z, a)$ pair as well as their corresponding theoretical limits $Q_{\piexpl}(z, a)$ (computed using Theorem~\ref{thm:QL-POMDP}) are shown in Fig.~\ref{fig:rasql-plots}. The salient features of these results are as follows:
\begin{itemize}
    \item RASQL converges to the theoretical limit predicted by Theorem~\ref{thm:QL-POMDP}.
    \item The limit $Q_{\piexpl}$ depends on the exploration policy $\piexpl$.
\end{itemize}

Thus, it can be seen from this example that we can precisely characterize the limits of convergence when using regularized Q-learning with an agent-state-based representation.

\begin{figure}
    \centering
    \includegraphics[width=\linewidth]{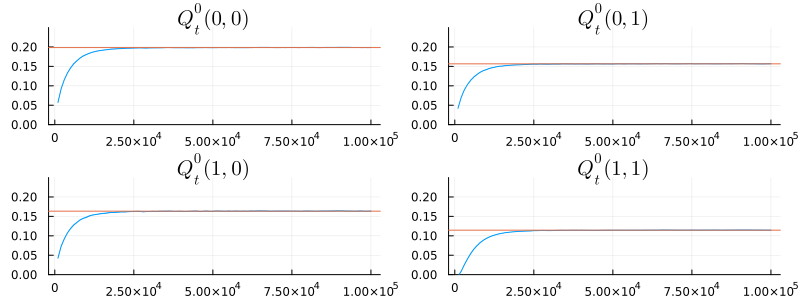}
    \caption{RASQL convergence: Q-values vs. number of iterations. Blue: RASQL iterates, Red: Theoretical limit from Theorem~\ref{thm:QL-POMDP}.}
    \label{fig:rasql-plots}
\end{figure}

\subsection{Regularized periodic agent-state-based Q-learning (RePASQL) experiment}

Similar to the RASQL experiment, we fix the agent state to be the observation of the agent, i.e., $z_t=y_t$. Consider the following fixed periodic exploration policy for period $L=2$:
\[
\piexpl^{0}(a \mid z) = \MATRIX{0.2 & 0.8 \\ 0.8 & 0.2}, \quad
\piexpl^{1}(a \mid z) = \MATRIX{0.8 & 0.2 \\ 0.2 & 0.8}.
\]

Using $\piexpl^\ell$, we run $25$ random seeds on the given POMDP and we perform the RePASQL update \eqref{eq:reg-PASQL} with a regularization coefficient $(\beta) = 1.0$ for $10^5$ timesteps/iterations. We plot the median and quartiles from $25$ seeds of the iterates $\{Q^{\ell}_{t}(z, a)\}_{t \geq 1}$ for each $(\ell, z, a)$ pair as well as their corresponding theoretical limits $Q^{\ell}_{\piexpl}(z, a)$ (computed using Theorem~\ref{thm:PRASQL}) are shown in Fig.~\ref{fig:prasql-plots}. The salient features of these results are as follows:
\begin{itemize}
    \item RePASQL converges to the theoretical limit predicted by Theorem~\ref{thm:PRASQL}.
    \item The limits $\{Q^{\ell}_{\piexpl}\}_{\ell \in [L]}$ depend on the periodic exploration policy $\{\piexpl^{\ell}\}_{\ell \in [L]}$.
\end{itemize}
Thus, it can be seen from this example that we can precisely characterize the limits of convergence.
\begin{figure}
    \centering
    \includegraphics[width=\linewidth]{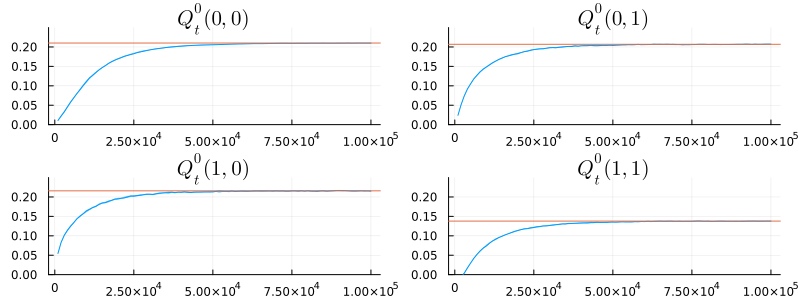}
    \includegraphics[width=\linewidth]{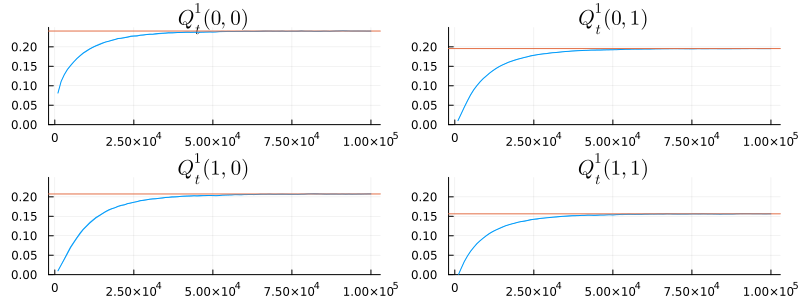}
    \caption{
    RePASQL convergence: Q-values vs. number of iterations. Blue: RePASQL iterates, Red: Theoretical limit from Theorem~\ref{thm:PRASQL}
    }
    \label{fig:prasql-plots}
\end{figure}

\section{Conclusions}

In this work, we present theoretical results on the convergence of regularized agent-state-based Q-learning (RASQL) under some standard assumptions from the literature. In particular, we show that: $1)$ RASQL converges and $2)$ we characterize the solution that RASQL converges to as a function of the model parameters and the choice of exploration policy. We illustrate these ideas on a small POMDP example and show that the Q-learning iterates of RASQL matches with the calculated theoretical limit. We also generalize these ideas to the periodic setting and demonstrate the theoretical and empirical convergence of RePASQL. Thus, in doing so we are able to understand how regularization works when combined with Q-learning for POMDPs that have an agent state that is not an information state.

A noteworthy issue with RASQL/RePASQL is that it inherits the limitations of its predecessor approaches of ASQL and PASQL. In particular, while we are able to prove convergence and characterize the converged solution in RASQL/RePASQL, we cannot guarantee the convergence to the optimal agent-state-based solution and this largely depends on the choice of exploration policy and the POMDP dynamics. Even so, seeing how regularization is an important component in several empirical works concerning POMDPs with agent states that are not an information state, we find it useful to establish some useful theoretical properties on the convergence of such algorithms.

\bibliographystyle{IEEEtran}
\bibliography{references}  

\cleardoublepage

\appendix

\subsection{Proof of Theorem~\ref{thm:QL-POMDP}} \label{sec:RASQL-proof}

The proof argument for Theorem~\ref{thm:QL-POMDP} is similar to the proof argument given in \cite{Jaakkola1994,Kara2022,Kara2024,sinha2024periodic}. 

Define an error function between the converged value and the Q-learning iteration $\Delta_{t+1} \coloneqq Q_{t+1} - Q_{\piexpl}$. Then, combine \eqref{eq:reg-QL}, \eqref{eq:pseudo-mdp-dp-q} and \eqref{eq:artificial-MDP} as follows for all $(z,a)$.
\begin{align}
    &\Delta_{t+1}(z, a) = Q_{t+1}(z, a) - Q_{\piexpl}(z, a) \nonumber \\
    &\quad= (1 - \alpha_t(z, a)) \Delta_t(z, a) \nonumber \\
    &\qquad + \alpha_t(z, a) \left[ U^{0}_t(z, a) + U^{1}_t(z, a) + U^{2}_t(z, a) \right], \label{eq:QL-update-diff}
\end{align}
where
\begin{align*}
\useshortskip
    & U^{0}_t(z, a) \coloneqq \left[ r(S_t, A_t) - r_{\piexpl} (z, a) \right] \IND_{\{ Z_t = z, A_t = a\}},\\
    & U^{1}_t(z, a) \coloneqq \left[ \vphantom{\sum_{A}} \gamma \regconj(Q_{\piexpl}(Z_{t+1}, \cdot)) \right. \\
    & - \left. \gamma \sum_{z' \in \ALPHABET Z} P_{\piexpl} (z' \mid z, a) \regconj(Q_{\piexpl}(z', \cdot)) \right] \IND_{\{ Z_t = z, A_t = a\}}, \\
    & U^{2}_t(z, a) \coloneqq \\
    &\quad \left[ \gamma \regconj (Q_{t}(Z_{t+1}, \cdot)) - \gamma \regconj(Q_{\piexpl}(Z_{t+1}, \cdot)) \right]
    \IND_{\{ Z_t = z, A_t = a\}}.
\end{align*}

Note that we are adding the term $\gamma \regconj(Q_{\piexpl}(Z_{t+1}, \cdot)) \IND_{\{ Z_t = z, A_t = a\}}$ in $U^{1}_t(z, a)$ and subtracting it in $U^{2}_t(z, a)$. We can now view \eqref{eq:QL-update-diff} as a linear system with state $\Delta_{t}$ and three inputs $U^{0}_t(z, a), U^{1}_t(z, a)$ and $U^{2}_t(z, a)$. Using the linearity, we can now split the state into three components $\Delta_{t+1} = X^{0}_{t+1} + X^{1}_{t+1} + X^{2}_{t+1}$, where the components evolve as follows for $i \in \{0, 1, 2\}$:
\begin{align*}
    X^{i}_{t+1} (z, a) = (1 - \alpha_t(z, a)) X^{i}_{t} (z, a) + \alpha_t(z, a) U^{i}_t(z, a).
\end{align*}
We will now separately show each $\NORM{X^{i}_{t}} \to 0$.

\paragraph{Convergence of component $X^{0}_t$}

The proof for the convergence of component $X^{0}_t$ is similar to that given in \cite{sinha2024periodic}.

\paragraph{Convergence of component $X^{1}_t$}

The proof for the convergence of component $X^{1}_t$ is based on the argument given in \cite{sinha2024periodic}. Let $W_t$ denote the tuple $(S_t,Z_t,A_t,S_{t+1},Z_{t+1},A_{t+1})$. Note that $\{W_t\}_{t\ge1}$ is a Markov chain and converges to a limiting distribution $\bar \zeta_{\piexpl}$, where 
\begin{align*}
   &\bar \zeta_{\piexpl}(s,z,a,s',z',a') \\
   &\quad = \zexpl(s,z,a) \sum_{y' \in \ALPHABET Y} P(s',y'|s,a) \IND_{\{z' = \phi(z,y',a)\}} \piexpl(a'|z'). 
\end{align*}
We use $\bar \zeta_{\piexpl}(s,z,a, \mathcal S, \mathcal Z, \mathcal A)$ to denote the marginalization over the ``future states'' and a similar notation for other marginalizations. Note that $\bar \zeta_{\piexpl}(s,z,a, \mathcal S, \mathcal Z, \mathcal A) = \zexpl(s,z,a)$. 

Define $V_{t}$ as the value function associated with $Q_t$, i.e., $V_t (z) \coloneqq \regconj(Q_t(z, \cdot))$. 
Fix $(z_\circ,a_\circ) \in \times \ALPHABET Z \times \ALPHABET A$ and define
\begin{align}
    h_P( W_t; z_\circ,a_\circ) &= \Bigl[ \gamma V_{\piexpl}(Z_{t+1}) - \nonumber \\
        &\gamma \sum_{\bar z \in \ALPHABET Z} P_{\piexpl}(\bar z|z_\circ,a_\circ) V_{\piexpl}(\bar z) \Bigr] \IND_{\{Z_t = z_\circ, A_t = a_\circ \}}.
\end{align}
Then the process $\{ X^1_t (z, a) \}_{t \geq 1}$ is given by the stochastic iteration
\begin{align}
    X^1_{t+1} (z_\circ, a_\circ) &= (1 - \alpha_t(z_\circ,a_\circ)) X^1_t(z_\circ, a_\circ) \nonumber \\
    &\quad + \alpha_t(z_\circ,a_\circ)  h_P( W_t ; z_\circ,a_\circ).
\end{align}
As argued earlier, the process $\{W_t\}_{t \ge 1}$ is a Markov chain.
Due to \autoref{ass:lr}, the learning rate $\alpha_t(z_\circ, a_\circ)$ is measurable with respect to the sigma-algebra generated by $(Z_{1:t}, A_{1:t})$ and is therefore also measurable with respect to the sigma-algebra generated by $W_{1:t}$. Thus, the learning rates $\{ \alpha_t(z_\circ, a_\circ) \}_{t \geq 1}$ satisfy the conditions of Theorem~$2.7$ from \cite{prashanth2024gradient}. Therefore, the theorem implies that  
$\{ X^1_t (z_\circ, a_\circ) \}_{t \geq 1}$ converges a.s. to the following limit
\begin{align*}
    & \lim_{t \to \infty} X^1_t (z_\circ, a_\circ) = \sum_{\substack{ s,z,a \in \ALPHABET S \times \ALPHABET Z \times \ALPHABET A \\ s',z',a' \in  \ALPHABET S \times \ALPHABET Z \times \ALPHABET A}}
    \bar \zeta_{\piexpl}(s,z,a,s', z', a') \\
    &\quad 
    h_P( s,z,a, s', z', a' ; z_\circ,a_\circ) \\
    %
    &= \gamma \biggl[ \sum_{z' \in \ALPHABET Z } 
       \bar \zeta_{\piexpl}( \mathcal S, z_\circ,a_\circ, \mathcal S, z',  \mathcal A) V_{\piexpl}(z') \biggr] 
       \\
       &\quad -
       \biggl[ \gamma \bar \zeta_{\piexpl}( \mathcal S, z_\circ,a_\circ, \mathcal S, \mathcal Z, \mathcal A) 
       \sum_{\bar z \in \ALPHABET Z} P_{\piexpl}(\bar z|z_\circ,a_\circ) V_{\piexpl}(\bar z) \biggr] 
    \\
    &= 0 
\end{align*}
where the last step follows from the fact that $\bar \zeta_{\piexpl}( \ALPHABET S, z_\circ,a_\circ, \ALPHABET S, \mathcal Z, \mathcal A) = \zexpl(z_\circ,a_\circ)$ and $\bar \zeta_{\piexpl}( \ALPHABET S, z_\circ,a_\circ, \ALPHABET S, z',  \ALPHABET A) = \zexpl(z_\circ,a_\circ) P_{\piexpl}(z'|z_\circ,a_\circ)$.




\paragraph{Convergence of component $X^{2}_t$} \label{subsec:convg-x2t}

    The convergence of the $X^{2}_t$ component is based on~\cite{Kara2022,sinha2024periodic} but requires some additional considerations due to the regularization term. We start by defining:
\begin{align*}
    \pi_t (\cdot \mid z) &= \argmax_{\xi \in \Delta(\ALPHABET A)} \sum_{a \in \ALPHABET A} \xi(a) Q_{t}(z, a) - \reg(\xi)\\
    \pi^{\star} (\cdot \mid z) &= \argmax_{\xi \in \Delta(\ALPHABET A)} \sum_{a \in \ALPHABET A} \xi(a) Q_{\piexpl}(z, a) - \reg(\xi).
\end{align*}

In the previous steps, we have shown that $\NORM{X^i_t} \to 0$ a.s., for $i \in \{0, 1\}$. Thus, we have that $\NORM{X^0_t + X^1_t} \to 0$ a.s. Arbitrarily fix an $\epsilon > 0$. Therefore, there exists a set $\Omega^1$ of measure one and a constant $T(\omega, \epsilon)$ such that for $\omega \in \Omega^1$, all $t > T(\omega, \epsilon)$, and  $(z,a) \in \times \ALPHABET Z \times \ALPHABET A$, we have 
\begin{equation}\label{eq:X1-eps}
    X^0_t(z,a) + X^1_t(z,a) < \epsilon.
\end{equation}

Now pick a constant $C$ such that 
\begin{equation}\label{eq:C}
    \kappa \coloneqq \gamma \left( 1 + \frac 1C \right) < 1
\end{equation}
Suppose for some $t > T(\omega, \epsilon)$, $\NORM{X^2_t} > C \epsilon$. Then, for $(z,a) \in \ALPHABET Z \times \ALPHABET A$, 
\begin{subequations}
\begin{align}
    &U^2_t(z,a) = \gamma V_t(Z_{t+1}) - \gamma V_{\piexpl}(Z_{t+1}) \\
    &= \gamma  \regconj(Q_{t}(Z_{t+1}, \cdot)) - \gamma  \regconj(Q_{\piexpl}(Z_{t+1}, \cdot)) \\
    &= \gamma \left[ \sum_{a \in \ALPHABET A} \pi_t(a \mid Z_{t+1}) Q_{t}(Z_{t+1}, a) - \reg(\pi_t(\cdot \mid Z_{t+1})) - \right. \nonumber \\
    &\quad \left.  \sum_{a \in \ALPHABET A} \pi^{\star}(a \mid Z_{t+1}) Q_{\piexpl}(Z_{t+1}, a) + \reg(\pi^{\star}(\cdot \mid Z_{t+1})) \right] 
    \displaybreak[1] \\
    &\stackrel{(a)}\leq \gamma \left[ \sum_{a \in \ALPHABET A} \pi_t(a \mid Z_{t+1}) Q_{t}(Z_{t+1}, a) - \reg(\pi_t(\cdot \mid Z_{t+1})) - \right. \nonumber \\
    &\quad \left. \sum_{a \in \ALPHABET A}\pi_t(a \mid Z_{t+1}) Q_{\piexpl}(Z_{t+1}, a) + \reg(\pi_t(\cdot \mid Z_{t+1})) \right] \\
    &\leq \gamma \sum_{a \in \ALPHABET A} \pi_t(a \mid Z_{t+1}) \bigl\lvert  Q_{t}(Z_{t+1}, a) - Q_{\piexpl}(Z_{t+1}, a) \bigr\rvert 
    \displaybreak[1] \\
    &\stackrel{(b)}\leq \gamma \lVert Q_{t} -  Q_{\piexpl} \rVert = \gamma \NORM{\Delta_t} \\
    &\le \gamma \NORM{X^0_t + X^1_t} + \gamma \NORM{X^2_t} \label{eq:U-bd2}
    \displaybreak[1]\\
    &\stackrel{(c)} \le \gamma \epsilon + \gamma \NORM{X^2_t} \label{eq:U-bd2-p2}
    \displaybreak[1] \\
    &\stackrel{(d)}\le \gamma \left( 1 + \frac 1C \right) \NORM{X^2_t}
    \stackrel{(e)} = \kappa \NORM{X^2_t}
    \stackrel{(e)} < \NORM{X^2_t},
    \label{eq:U-bd3}
\end{align}
\end{subequations}
where $(a)$ follows from the fact that we replace the argmax $\pi^{\star}$ with a different argument $\pi_t$ in the second term, $(b)$ follows from maximizing over all realizations of $Z_{t+1}$ and $a \in \ALPHABET A$, $(c)$ follows from \eqref{eq:X1-eps}, $(d)$ follows from $\NORM{X^2_t} > C \epsilon$, $(e)$ follows from \eqref{eq:C}. Thus, for any $t > T(\omega, \epsilon)$ and $\NORM{X^2_t} > C \epsilon$:
\begin{align*}
    X^2_{t+1}(z,a)
    &=
    (1 - \alpha_t(z,a)) X^2_t(z,a)
    +\\ 
    &\quad \alpha_t(z,a) U^2_t(z,a)
    < 
    \NORM{X^2_t} \\
    \implies \NORM{X^2_{t+1}} & < \NORM{X^2_t}.
\end{align*}

Hence, when $\NORM{X^2_t} > C\epsilon$, it decreases monotonically with time. Hence, there are two possibilities: 
either 
\begin{enumerate} 
    \item $\NORM{X^2_t}$ always remains above $C\epsilon$; or
    \item it goes below $C\epsilon$ at some stage.
\end{enumerate}
We consider these two possibilities separately.

\textbf{Possibility (i): \texorpdfstring{$\NORM{X^2_t}$ always remains above $C\epsilon$}{}}

\newcommand\M[1]{\NORM{M^{(#1)}_t}}



We will now prove that $\NORM{X^2_t}$ cannot remain above $C\epsilon$ forever.  The proof is by contradiction.
Suppose $\NORM{X^2_t}$ remains above $C\epsilon$ forever. As argued earlier, this implies that $\NORM{X^2_t}$, $t \ge T(\omega,\epsilon)$, is a strictly decreasing sequence, so it must be bounded from above. Let $B^{(0)}$ be such that $\NORM{X^2_t} \le B^{(0)}$ for all $t \ge T(\omega,\epsilon)$. Eq.~\eqref{eq:U-bd3} implies that $\NORM{U^2_t} < \kappa B^{(0)}$. Then, we have for all $(z,a) \in \ALPHABET Z \times \ALPHABET A$ that 
\begin{align*}
    X^2_{t+1} (z, a) &\le (1 - \alpha_t(z,a)) \NORM{X^2_t} + \alpha_t(z,a) \NORM{U^2_t} \\
    &< (1 - \alpha_t(z,a)) \NORM{X^2_t} + \alpha_t(z,a) \kappa \NORM{X^2_t}
\end{align*}
which implies that $\NORM{X^2_t} \le \M0$, where $\{M^{(0)}_t\}_{t \ge T(\omega,\epsilon)}$ is a sequence given by
\begin{align}
    M^{(0)}_{t+1}(z,a)  \le (1 - \alpha_t(z,a)) M^{(0)}_t(z,a) + \alpha_t(z,a) \kappa B^{(0)}.
    \label{eq:M0}
\end{align}
Theorem~$2.4$ from \cite{prashanth2024gradient} implies that $M^{(0)}_t(z,a) \to \kappa B^{(0)}$ and hence $\M0 \to \kappa B^{(0)}$. Now pick an arbitrary $\bar \epsilon \in (0, (1-\kappa) C\epsilon)$. Thus, there exists a time $T^{(1)} = T^{(1)}(\omega, \epsilon, \bar \epsilon)$ such that for all $t > T^{(1)}$, $\M0 \le B^{(1)} \coloneqq \kappa B^{(0)} + \bar \epsilon$. Since $\NORM{X^{2}_t}$ is bounded by $\M0$, this implies that for all $t > T^{(1)}$, $\NORM{X^2_t} \le B^{(1)}$ and, by~\eqref{eq:U-bd3}, $\NORM{U^2_t} \le \kappa B^{(1)}$. By repeating the above argument, there exists a time $T^{(2)}$ such that for all $t \ge T^{(2)}$, 
\begin{equation}
\NORM{X^2_t} \le B^{(2)} \coloneqq \kappa B^{(1)} + \bar \epsilon = \kappa^2 B^{(0)} + \kappa \bar \epsilon + \bar \epsilon, 
\end{equation}
and so on. By~\eqref{eq:C}, $\kappa < 1$ and $\bar\epsilon$ is chosen to be less than $C\epsilon$. So eventually, $B^{(m)} \coloneqq \kappa^m B^{(0)} + \kappa^{m-1} \bar \epsilon + \cdots + \bar \epsilon$ must get below $C\epsilon$ for some $m$, contradicting the assumption that $\NORM{X^2_t}$ remains above $C\epsilon$ forever.

\textbf{Possibility (ii): \texorpdfstring{$\NORM{X^2_t}$ goes below $C\epsilon$ at some stage}{}}

Suppose that there is some $t > T(\omega, \epsilon)$ such that $\NORM{X^2_t} < C \epsilon$. Then~\eqref{eq:U-bd2}, \eqref{eq:U-bd2-p2} and \eqref{eq:C} imply that
\begin{equation}
    \NORM{U^2_t} \le \gamma \NORM{X^0_t + X^1_t} + \gamma \NORM{X^2_t}
    \le \gamma \epsilon + \gamma C \epsilon < C\epsilon.
\end{equation}
Therefore,
\begin{equation}
    X^2_{t+1}(z,a) \le 
    (1 - \alpha_t(z,a)) \NORM{X^2_t} + \alpha_t(z,a) \NORM{U^2_t}
    < C \epsilon
\end{equation}
where the last inequality uses the fact that both $\NORM{U^2_t}$ and $\NORM{X^2_{t+1}}$ are both below $C\epsilon$. Thus, we have that 
\begin{equation}
    X^2_{t+1}(z,a) < C\epsilon.
\end{equation}
Hence, once $\NORM{ X^2_{t+1} }$ goes below $C\epsilon$, it stays there.

\paragraph{Implication}
We have show that for sufficiently large $t > T(\omega, \epsilon)$, $ X^2_{t}(z,a) < C \epsilon$. Since $\epsilon$ is arbitrary, this means that for all realizations $\omega \in \Omega^1$, $\NORM{X^2_{t}} \to 0$. Thus, 
\begin{equation}\label{eq:X2-limit}
    \lim_{t \to \infty} \NORM{X^2_{t}} = 0, \quad a.s.
\end{equation}

\textbf{Putting everything together}

Recall that we initially defined $\Delta_t = Q_t - Q_{\piexpl}$ and we split $\Delta_t = X^0_t + X^1_t + X^2_t$. Steps~$a)$ and $b)$ together show that $\NORM{X^0_t + X^1_t} \to 0$, a.s.\ and Step $c)$ \eqref{eq:X2-limit} shows us that $\NORM{X^2_{t}} \to 0$, a.s. Thus, by the triangle inequality, 
\begin{equation}
\lim_{t \to \infty} \NORM{\Delta_t} \le 
\lim_{t \to \infty} \NORM{X^0_t + X^1_t} 
+
\lim_{t \to \infty} \NORM{X^2_t} 
= 0,
\end{equation}
which establishes that $Q_t \to Q_{\piexpl}$, a.s.

\subsection{Proof of Theorem~\ref{thm:PRASQL}} \label{sec:PRASQL-proof}

The proof follows a similar style used in \cite{sinha2024periodic}. Define an error function between the converged value and the Q-learning iteration $\Delta^{\ell}_{t+1} \coloneqq Q^{\ell}_{t+1} - Q^{\ell}_{\piexpl}$. Then, combine \eqref{eq:reg-PASQL}, \eqref{eq:pseudo-mdp-dp-q-periodic} and \eqref{eq:artificial-MDP-periodic} as follows for all $(z,a)$.
\begin{align}
    &\Delta^{\ell}_{t+1}(z, a) = Q^{\ell}_{t+1}(z, a) - Q^{\ell}_{\piexpl}(z, a) \nonumber \\
    &\quad= (1 - \alpha_t(z, a)) \Delta^{\ell}_t(z, a) \nonumber \\
    &\qquad + \alpha_t(z, a) \left[ U^{\ell, 0}_t(z, a) + U^{\ell, 1}_t(z, a) + U^{\ell, 2}_t(z, a) \right], \label{eq:QL-update-diff-periodic}
\end{align}
where
\begin{align*}
\useshortskip
    & U^{\ell, 0}_t(z, a) \coloneqq \left[ r(S_t, A_t) - r^{\ell}_{\piexpl} (z, a) \right] \IND_{\{ Z_t = z, A_t = a\}},\\
    & U^{\ell, 1}_t(z, a) \coloneqq \left[ \vphantom{\sum_{A}} \gamma \regconj(Q^{\MOD{\ell+1}}_{\piexpl}(Z_{t+1}, \cdot)) \right. \\
    & - \left. \gamma \sum_{z' \in \ALPHABET Z} P^{\ell}_{\piexpl} (z' \mid z, a) \regconj(Q^{\MOD{\ell+1}}_{\piexpl}(z', \cdot)) \right] \IND_{\{ Z_t = z, A_t = a\}}, \\
    & U^{\ell, 2}_t(z, a) \coloneqq \left[ \gamma \regconj (Q^{\MOD{\ell+1}}_{t}(Z_{t+1}, \cdot)) \right. \\
    & - \left. \gamma \regconj(Q^{\MOD{\ell+1}}_{\piexpl}(Z_{t+1}, \cdot)) \right]
    \IND_{\{ Z_t = z, A_t = a\}}.
\end{align*}
Note that we are adding the term $\gamma \regconj(Q^{\MOD{\ell+1}}_{\piexpl}(Z_{t+1}, \cdot)) \IND_{\{ Z_t = z, A_t = a\}}$ in $U^{\ell, 1}_t(z, a)$ and subtracting it in $U^{\ell, 2}_t(z, a)$. We can now view \eqref{eq:QL-update-diff-periodic} as a linear system with state $\Delta^{\ell}_{t}$ and three inputs $U^{\ell, 0}_t(z, a), U^{\ell, 1}_t(z, a)$ and $U^{\ell, 2}_t(z, a)$. Using the linearity, we can now split the state into three components $\Delta^{\ell}_{t+1} = X^{\ell, 0}_{t+1} + X^{\ell, 1}_{t+1} + X^{\ell, 2}_{t+1}$, where the components evolve as follows for $i \in \{0, 1, 2\}$:
\begin{align*} 
    X^{\ell, i}_{t+1} (z, a) = (1 - \alpha_t(z, a)) X^{\ell, i}_{t} (z, a) + \alpha_t(z, a) U^{\ell, i}_t(z, a).
\end{align*}
We will now separately show each $\NORM{X^{\ell, i}_{t}} \to 0$.
\paragraph{Convergence of component $X^{\ell, 0}_t$}
The proof for the convergence of component $X^{\ell, 0}_t$ is similar to that given in \cite{sinha2024periodic}. The only difference from the RASQL proof of Theorem~\ref{thm:QL-POMDP} is that the convergence has to established for each $\ell \in [L]$ in $\NORM{X^{\ell, i}_{t}} \to 0$. Note that this case is identical to the periodic case of \cite{sinha2024periodic}, since the component $\NORM{X^{\ell, i}_{t}}$ does not involve any of the regularized terms. \\
The main result that is applied here is proposition~$4$ from \cite{sinha2024periodic}, which establishes the exact convergence of $\NORM{X^{\ell, i}_{t}}$ when the underlying Markov chain is periodic.
\paragraph{Convergence of component $X^{\ell, 1}_t$}
The proof for the convergence of component $X^{\ell, 1}_t$ is based on the argument given in \cite{sinha2024periodic}. Let $W_t$ denote the tuple $(S_t,Z_t,A_t,S_{t+1},Z_{t+1},A_{t+1})$. Note that $\{W_t\}_{t\ge1}$ is a periodic Markov chain and converges to a periodic limiting distribution $\bar \zeta^{\ell}_{\piexpl}$, where 
\begin{align*}
   &\bar \zeta^{\ell}_{\piexpl}(s,z,a,s',z',a') \\
   &\quad = \zexpl^{\ell}(s,z,a) \sum_{y' \in \ALPHABET Y} P(s',y'|s,a) \IND_{\{z' = \phi(z,y',a)\}} \piexpl(a'|z'). 
\end{align*}
We use $\bar \zeta^{\ell}_{\piexpl}(s,z,a, \mathcal S, \mathcal Z, \mathcal A)$ to denote the marginalization over the ``future states'' and a similar notation for other marginalizations. Note that $\bar \zeta^{\ell}_{\piexpl}(s,z,a, \mathcal S, \mathcal Z, \mathcal A) = \zexpl^{\ell}(s,z,a)$. \\
Define $V^{\MOD{\ell+1}}_{t}$ as the value function associated with $Q^{\MOD{\ell+1}}_t$, i.e., $V^{\MOD{\ell+1}}_t (z) \coloneqq \regconj(Q^{\MOD{\ell+1}}_t(z, \cdot))$. 
Fix $(z_\circ,a_\circ) \in \times \ALPHABET Z \times \ALPHABET A$ and define
\begin{align}
    h_P( W_t; \ell, z_\circ,a_\circ) &= \Bigl[ \gamma V^{\MOD{\ell+1}}_{\piexpl}(Z_{t+1}) - \nonumber \\
        &\gamma \sum_{\bar z \in \ALPHABET Z} P^{\ell}_{\piexpl}(\bar z|z_\circ,a_\circ) V^{\MOD{\ell+1}}_{\piexpl}(\bar z) \Bigr] \IND_{\{Z_t = z_\circ, A_t = a_\circ \}}.
\end{align}
Then the process $\{ X^{\ell, 1}_t (z, a) \}_{t \geq 1}$ is given by the stochastic iteration
\begin{align}
    X^{\ell, 1}_{t+1} (z_\circ, a_\circ) &= (1 - \alpha^{\ell}_t(z_\circ,a_\circ)) X^{\ell, 1}_t(z_\circ, a_\circ) \nonumber \\
    &\quad + \alpha^{\ell}_t(z_\circ,a_\circ)  h_P( W_t ; \ell, z_\circ,a_\circ).
\end{align}
As mentioned earlier, the process $\{W_t\}_{t \ge 1}$ is a periodic Markov chain.
From the periodic Markov chain result of proposition~$4$ from \cite{sinha2024periodic}, we have that:
$\{ X^{\ell, 1}_t (z_\circ, a_\circ) \}_{t \geq 1}$ converges a.s. to the following periodic limits
\begin{align*}
    & \lim_{t \to \infty} X^{\ell, 1}_t (z_\circ, a_\circ) = \sum_{\substack{ s,z,a \in \ALPHABET S \times \ALPHABET Z \times \ALPHABET A \\ s',z',a' \in  \ALPHABET S \times \ALPHABET Z \times \ALPHABET A}}
    \bar \zeta^{\ell}_{\piexpl}(s,z,a,s', z', a') \\
    &\quad 
    h_P( s,z,a, s', z', a' ; \ell, z_\circ,a_\circ) \\
    %
    &= \gamma \biggl[ \sum_{z' \in \ALPHABET Z } 
       \bar \zeta^{\ell}_{\piexpl}( \mathcal S, z_\circ,a_\circ, \mathcal S, z',  \mathcal A) V^{\MOD{\ell+1}}_{\piexpl}(z') \biggr] 
       \\
       &\quad -
       \biggl[ \gamma \bar \zeta^{\ell}_{\piexpl}( \mathcal S, z_\circ,a_\circ, \mathcal S, \mathcal Z, \mathcal A) 
       \sum_{\bar z \in \ALPHABET Z} P^{\ell}_{\piexpl}(\bar z|z_\circ,a_\circ) V^{\MOD{\ell+1}}_{\piexpl}(\bar z) \biggr] 
    \\
    &= 0 
\end{align*}
where the last step follows from the fact that $\bar \zeta^{\ell}_{\piexpl}( \ALPHABET S, z_\circ,a_\circ, \ALPHABET S, \mathcal Z, \mathcal A) = \zexpl^{\ell}(z_\circ,a_\circ)$ and $\bar \zeta^{\ell}_{\piexpl}( \ALPHABET S, z_\circ,a_\circ, \ALPHABET S, z',  \ALPHABET A) = \zexpl^{\ell}(z_\circ,a_\circ) P^{\ell}_{\piexpl}(z'|z_\circ,a_\circ)$.
\paragraph{Convergence of component $X^{\ell, 2}_t$} \label{subsec:convg-x2t-periodic}
The convergence of the $X^{\ell, 2}_t$ component is based on~\cite{Kara2022,sinha2024periodic} but requires some additional considerations due to the regularization term. We start by defining:
\begin{align*}
    \pi^{\ell}_t (\cdot \mid z) &= \argmax_{\xi \in \Delta(\ALPHABET A)} \sum_{a \in \ALPHABET A} \xi(a) Q^{\MOD{\ell+1}}_{t}(z, a) - \reg(\xi)\\
    \pi^{\ell, \star} (\cdot \mid z) &= \argmax_{\xi \in \Delta(\ALPHABET A)} \sum_{a \in \ALPHABET A} \xi(a) Q^{\MOD{\ell+1}}_{\piexpl}(z, a) - \reg(\xi).
\end{align*}
In the previous steps, we have shown that $\NORM{X^{\ell,i}_t} \to 0$ a.s., for $i \in \{0, 1\}$. Thus, we have that $\NORM{X^{\ell, 0}_t + X^{\ell, 1}_t} \to 0$ a.s. Arbitrarily fix an $\epsilon > 0$. Therefore, there exists a set $\Omega^1$ of measure one and a constant $T(\omega, \epsilon)$ such that for $\omega \in \Omega^1$, all $t > T(\omega, \epsilon)$, and  $(z,a) \in \times \ALPHABET Z \times \ALPHABET A$, we have 
\begin{equation}\label{eq:X1-eps-periodic}
    X^{\ell, 0}_t(z,a) + X^{\ell, 1}_t(z,a) < \epsilon.
\end{equation}
Now pick a constant $C$ such that 
\begin{equation}\label{eq:C-periodic}
    \kappa \coloneqq \gamma \left( 1 + \frac 1C \right) < 1
\end{equation}
Suppose for some $t > T(\omega, \epsilon)$, $\NORM{X^{\ell, 2}_t} > C \epsilon$. Then, for $(\ell, z,a) \in L \times \ALPHABET Z \times \ALPHABET A$, 
\begin{subequations}
\begin{align}
    &U^{\ell, 2}_t(z,a) = \gamma V^{\MOD{\ell+1}}_t(Z_{t+1}) - \gamma V^{\MOD{\ell+1}}_{\piexpl}(Z_{t+1}) \\
    &= \gamma  \regconj(Q^{\MOD{\ell+1}}_{t}(Z_{t+1}, \cdot)) - \gamma  \regconj(Q^{\MOD{\ell+1}}_{\piexpl}(Z_{t+1}, \cdot)) \\
    &= \gamma \left[ \sum_{a \in \ALPHABET A} \pi^{\ell}_t(a \mid Z_{t+1}) Q^{\MOD{\ell+1}}_{t}(Z_{t+1}, a) - \reg(\pi^{\ell}_t(\cdot \mid Z_{t+1})) - \right. \nonumber \\
    &\quad \left.  \sum_{a \in \ALPHABET A} \pi^{\ell, \star}(a \mid Z_{t+1}) Q^{\MOD{\ell+1}}_{\piexpl}(Z_{t+1}, a) + \reg(\pi^{\ell, \star}(\cdot \mid Z_{t+1})) \right] 
    \displaybreak[1] \\
    &\stackrel{(a)}\leq \gamma \left[ \sum_{a \in \ALPHABET A} \pi^{\ell}_t(a \mid Z_{t+1}) Q^{\MOD{\ell+1}}_{t}(Z_{t+1}, a) - \reg(\pi^{\ell}_t(\cdot \mid Z_{t+1})) - \right. \nonumber \\
    &\quad \left. \sum_{a \in \ALPHABET A}\pi^{\ell}_t(a \mid Z_{t+1}) Q^{\MOD{\ell+1}}_{\piexpl}(Z_{t+1}, a) + \reg(\pi^{\ell}_t(\cdot \mid Z_{t+1})) \right] \\
    &\leq \gamma \sum_{a \in \ALPHABET A} \pi^{\ell}_t(a \mid Z_{t+1}) \bigl\lvert Q^{\MOD{\ell+1}}_{t}(Z_{t+1}, a) - Q^{\MOD{\ell+1}}_{\piexpl}(Z_{t+1}, a) \bigr\rvert 
    \displaybreak[1] \\
    &\stackrel{(b)}\leq \gamma \lVert Q^{\MOD{\ell+1}}_{t} -  Q^{\MOD{\ell+1}}_{\piexpl} \rVert = \gamma \NORM{\Delta^{\ell}_t} \\
    &\le \gamma \NORM{X^{\ell, 0}_t + X^{\ell, 1}_t} + \gamma \NORM{X^{\ell, 2}_t} \label{eq:U-bd2-periodic}
    \displaybreak[1]\\
    &\stackrel{(c)} \le \gamma \epsilon + \gamma \NORM{X^{\ell, 2}_t} \label{eq:U-bd2-p2-periodic}
    \displaybreak[1] \\
    &\stackrel{(d)}\le \gamma \left( 1 + \frac 1C \right) \NORM{X^{\ell, 2}_t}
    \stackrel{(e)} = \kappa \NORM{X^{\ell, 2}_t}
    \stackrel{(e)} < \NORM{X^{\ell, 2}_t},
    \label{eq:U-bd3-periodic}
\end{align}
\end{subequations}
where $(a)$ follows from the fact that we replace the argmax $\pi^{\ell, \star}$ with a different argument $\pi^{\ell}_t$ in the second term, $(b)$ follows from maximizing over all realizations of $Z_{t+1}$ and $a \in \ALPHABET A$, $(c)$ follows from \eqref{eq:X1-eps-periodic}, $(d)$ follows from $\NORM{X^{\ell, 2}_t} > C \epsilon$, $(e)$ follows from \eqref{eq:C-periodic}. Thus, for any $t > T(\omega, \epsilon)$ and $\NORM{X^{\ell, 2}_t} > C \epsilon$:
\begin{align*}
    X^{\ell, 2}_{t+1}(z,a)
    &=
    (1 - \alpha^{\ell}_t(z,a)) X^{\ell, 2}_t(z,a)
    +\\ 
    &\quad \alpha^{\ell}_t(z,a) U^{\ell, 2}_t(z,a)
    < 
    \NORM{X^{\ell, 2}_t} \\
    \implies \NORM{X^{\ell, 2}_{t+1}} & < \NORM{X^{\ell, 2}_t}.
\end{align*}
Hence, when $\NORM{X^{\ell, 2}_t} > C\epsilon$, it decreases monotonically with time. Hence, there are two possibilities: 
either 
\begin{enumerate} 
    \item $\NORM{X^{\ell, 2}_t}$ always remains above $C\epsilon$; or
    \item it goes below $C\epsilon$ at some stage.
\end{enumerate}
These two cases must be considered separately. The proof follows exactly the same steps in the proof of theorem~\ref{thm:QL-POMDP} given in appendix~\ref{sec:RASQL-proof}, which finally gives us:
\begin{equation}\label{eq:X2-limit-periodic}
    \lim_{t \to \infty} \NORM{X^{\ell, 2}_{t}} = 0, \quad a.s.
\end{equation}
\\
\textbf{Putting everything together}
Recall that we initially defined $\Delta^{\ell}_t = Q^{\ell}_t - Q^{\ell}_{\piexpl}$ and we split $\Delta^{\ell}_t = X^{\ell, 0}_t + X^{\ell, 1}_t + X^{\ell, 2}_t$. Steps~$a)$ and $b)$ together show that $\NORM{X^{\ell, 0}_t + X^{\ell, 1}_t} \to 0$, a.s.\ and Step $c)$ \eqref{eq:X2-limit-periodic} shows us that $\NORM{X^{\ell, 2}_{t}} \to 0$, a.s. Thus, by the triangle inequality, 
\begin{equation}
\lim_{t \to \infty} \NORM{\Delta_t} \le 
\lim_{t \to \infty} \NORM{X^{\ell, 0}_t + X^{\ell, 1}_t} 
+
\lim_{t \to \infty} \NORM{X^{\ell, 2}_t} 
= 0,
\end{equation}
which establishes that $Q^{\ell}_t \to Q^{\ell}_{\piexpl}$, a.s.

\end{document}